\documentclass[letterpaper, 10pt, journal, twoside]{IEEEtran}
\usepackage{amsmath,amsfonts, amsthm}
\usepackage{algorithmic}
\usepackage{algorithm}
\usepackage{array}
\usepackage[caption=false,font=normalsize,labelfont=sf,textfont=sf]{subfig}
\usepackage{textcomp}
\usepackage{stfloats}
\usepackage{url}
\usepackage{verbatim}
\usepackage{graphicx}
\usepackage{cite}

\usepackage{amssymb}
\usepackage{array}
\usepackage{booktabs}
\usepackage{dsfont}
\usepackage{float}
\usepackage{caption}
\newtheorem{theorem}{Theorem}

\usepackage{graphicx}

\usepackage{enumitem}
\usepackage{multirow}
\usepackage[table]{xcolor}
\begin{document}

\title{Capacity-Aware Planning and Scheduling in Budget-Constrained Multi-Agent MDPs: A Meta-RL Approach}

\author{Manav Vora, Ilan Shomorony, Melkior Ornik 
\thanks{Manuscript received: May 20, 2025; Revised August 24, 2025; Accepted September 16, 2025.}
\thanks{This paper was recommended for publication by Chao-Bo Yan upon evaluation of the Associate Editor and Reviewers' comments. The work of M.V. and M.O. was supported by the Office of Naval Research (ONR) under Grant N00014-23-1-2505.
The work of I.S. was supported in part by the National Science Foundation (NSF) under Grant CCF-2046991.}%
\thanks{All authors are with the University of Illinois Urbana-Champaign, Urbana, IL 61801 USA
        {\tt\footnotesize mkvora2@illinois.edu, mornik@illinois.edu, ilans@illinois.edu}}%
\thanks{Digital Object Identifier (DOI): see top of this page.}
}

\markboth{IEEE Robotics and Automation Letters. Preprint Version. Accepted September 2025}
{Vora \MakeLowercase{\textit{et al.}}: Meta-RL for Constrained Multi-Agent MDPs}


\maketitle

\begin{abstract}
We study \emph{capacity- and budget-constrained multi-agent MDPs} (CB-MA-MDPs), a class that captures many maintenance and scheduling tasks in which each agent can irreversibly fail and a planner must decide \emph{(i)} \textit{when} to apply a restorative action and \emph{(ii)} \textit{which} subset of agents to treat in parallel.  The global budget limits the \emph{total} number of restorations, while the capacity constraint bounds the number of \emph{simultaneous} actions, turning naïve dynamic programming into a combinatorial search that scales exponentially with the number of agents.  
We propose a two-stage solution that remains tractable for large systems.  
First, a Linear Sum Assignment Problem (LSAP)-based grouping partitions the agents into $r$ disjoint sets ($r$\,=\,capacity) that maximise diversity in expected time-to-failure, allocating budget to each set proportionally.  
Second, a \textbf{meta-trained PPO} policy solves each sub-MDP, leveraging transfer across groups to converge rapidly.  
To validate our approach, we apply it to the problem of scheduling repairs for a large team of industrial robots, constrained by a limited number of repair technicians and a total repair budget. Our results demonstrate that the proposed method outperforms baseline approaches in terms of maximizing the average uptime of the robot team, particularly for large team sizes. Lastly, we confirm the scalability of our approach through a computational complexity analysis across varying numbers of robots and repair technicians.
\end{abstract}


\section{Introduction}\label{sec:intro}
\IEEEPARstart{M}{odern} production plants, data centres and energy networks operate fleets of independent units—robot arms \cite{borgi2017data}, electric-vehicle chargers \cite{Walraven16EV}, servers \cite{kibira2023degradation}—whose health degrades stochastically and must be restored through costly interventions.  Only a finite stock of parts or labour (\emph{budget}) and a limited number of parallel service slots (\emph{capacity}) are available at any time \cite{perlman2001setting}.  Such decision problems instantiate the constrained multi-agent MDP family \cite{DeNijs21Taxonomy}: each agent evolves independently until it reaches an absorbing failure state, while global constraints couple their actions.  Exact dynamic-programming or branch-and-bound methods quickly become infeasible because the admissible joint action set at each step contains $\binom{n}{r}$ subsets of agents \cite{toth2000optimization}.  Algorithms designed for \emph{budget-only} constraints \cite{vora2024solving, kalagarla2021sample} do not address the simultaneous-action limit imposed by capacity.

We present a two-stage framework that scales to thousands of agents.  First, a Linear Sum Assignment Problem (LSAP) partitioning step clusters the $n$ agents into exactly $r$ groups (where $r$ equals the capacity) by maximising diversity in expected time-to-failure.  The global budget is then apportioned proportionally across groups, converting the exponential joint action space into $r$ independent sub-MDPs.  Second, a \emph{meta-reinforcement-learning} stage trains a single Proximal Policy optimization (PPO) agent on a distribution of such sub-MDPs, enabling rapid adaptation to each group at run time. This two-step approach not only simplifies the capacity constraints by partitioning the agents into groups but also results in a more tractable learning process by focusing on smaller MDPs with reduced dimensionality, compared to a single large multi-agent MDP.

\textbf{Contributions.} (i) A formal model of capacity- and budget-constrained multi-agent MDPs with a single structural assumption—each agent has an absorbing failure state. (ii) An LSAP-based partitioning algorithm that transforms the exponential joint action space into $r$ parallel sub-MDPs. (iii) A meta-PPO solver that amortises learning across partitions and budget levels. (iv) Empirical evidence and a complexity analysis demonstrating scalability to large systems.

\section{Related Work}
\label{sec:prelims}
\subsection{Constrained Markov Decision Processes}\label{sec:cmdp}
A \emph{Constrained MDP} (CMDP) augments the classical MDP with one or more cumulative cost signals and requires any feasible policy to satisfy bounds on those costs \cite{altman2021constrained}.  Solvers typically rely on Lagrangian relaxations \cite{achiam2017cpo} or linear programming \cite{dolgov2005}; both scale poorly when many weakly coupled components share a \emph{joint} resource because the joint state space grows as \(\prod_{i=1}^{n}|\mathcal S^{i}|\) \cite{DeNijs21Taxonomy}.

\paragraph{Single-resource constraints}
Early scalable surrogates allocate a global budget across components either greedily or via convex relaxations \cite{vora2024solving}.  These methods assume that at most one resource couples the agents and that the constraint applies to the \emph{cumulative} number of costly actions.

\paragraph{Renewable budgets and multi-agent extensions}  
Consumption-MDPs handle resources that replenish over time (battery charge, fuel) but retain a single budget variable \cite{Blahoudek20consMDP}.  
In multi-agent settings, the same joint-budget coupling gives rise to \emph{constrained multi-agent MDPs} \cite{DeNijs21Taxonomy}; exact value iteration or MILP formulations remain feasible only for a few dozen agents, motivating the need for decomposition techniques explored in this paper.

\subsection{Capacity and Joint Budget–Capacity Constraints}\label{sec:capacity}
A \emph{capacity constraint} limits how many costly actions may be executed \emph{simultaneously}---e.g.\ at most \(r\) robots can be repaired in parallel. Linear-programming approximations for such limits exist \cite{haksar2018controlling} yet introduce \(O(|\mathcal{S}^i|^n|\mathcal{A}^i|^n)\) variables, becoming intractable for \(n\!\gg\!10\). Auction-based schedulers \cite{Gautier23MultiUnitAuction} allocate chance-constrained resources but require an inner planner for each bid—again a bottleneck at scale.  Our LSAP partitioning eliminates the \(\binom{n}{r}\) joint-action explosion by decomposing the problem into \(r\) sub-MDPs.
\subsection{Reinforcement Learning under Resource Constraints}\label{sec:rl}
Safe and budgeted RL algorithms enforce cumulative-cost limits by augmenting the reward with Lagrange multipliers \cite{achiam2017cpo} or by learning a two-dimensional $Q$-function that trades off return and cost \cite{carrara2019budgeted}.  
While effective for a \emph{single} resource, these methods assume the agent can choose any action in $\mathcal A$; they do not address the combinatorial action sets that arise when at most $r$ components may act in parallel.  
Attempts to embed such capacity constraints via action masking or enumeration quickly become impractical: the policy must explore $\binom{n}{r}$ joint moves, and training time scales super-linearly \cite{ray2019safetygym}.  

Multi-agent schedulers based on deep RL---for example DeepRM for datacentre job placement \cite{mao2016deeprm}---circumvent the explosion by restricting to small $n$ or by hand-crafting a fixed action order; performance degrades sharply once the fleet size grows beyond a few dozen units.

\smallskip
\noindent\textbf{Gap addressed.}  Prior CMDP, capacity-limited, and budgeted-RL approaches each tackle \emph{one} constraint but falter when both constraints coexist at scale.  The LSAP + meta-PPO framework proposed here is, to our knowledge, the first to handle joint budget–capacity limits for large systems with near-linear empirical runtime.
\section{Problem Formulation}\label{sec:prob}
In this work, we study capacity- and budget-constrained multi-agent Markov decision processes (CB-MA-MDPs).  
A CB-MA-MDP comprises \(n\) independent agents—e.g.\ robots or servers—whose states deteriorate stochastically and can be reset by a \emph{restorative action}.  
Two global resource limits couple the agents: (i) \emph{capacity}: at most \(r\) agents can be restored \emph{simultaneously} in any decision step, (ii) \emph{budget}: no more than \(B\) restorative actions may be executed over the entire planning horizon. At each step the planner must therefore select a subset of up to \(r\) agents from the \(n\) candidates, while ensuring the cumulative number of restorations never exceeds \(B\).  
The resulting joint-action space is combinatorial: the number of admissible action sequences grows super-exponentially in \(n\) and \(B\) \cite{papadimitriou1987complexity}, making naïve dynamic programming intractable.

\noindent\paragraph{State and action spaces}
For agent \(i\!\in\!\{1,\dots,n\}\) let the local state space be a finite set \(\mathcal S^{i}\) that contains a designated absorbing failure state \(s^{\mathcal F}\).  
The global state is the product \(s_k=(s^{1}_{k},\dots,s^{n}_{k})\in\mathcal S:=\prod_{i}\mathcal S^{i}\).  
Each agent chooses between two actions:  
\[
\mathcal A^{i}=\{d^{i},\,m^{i}\},
\]
where $d^i$ is an idle or ``do-nothing'' action and $m^i$ is a costly restoration action.
The joint action is \(a_k=(a^{1}_{k},\dots,a^{n}_{k})\in\mathcal A:=\prod_{i}\mathcal A^{i}\).

\noindent\paragraph{Transition kernel}
Because agents are physically independent, the global transition factorises:
\[
T(s_k,a_k,s_{k+1})=\prod_{i=1}^{n}T^{i}(s^{i}_{k},a^{i}_{k},s^{i}_{k+1}),
\]
where the local kernel \(T^{i}\) obeys only two structural rules:

\begin{enumerate}
\item \textbf{Absorption:} \(T^{i}(s^{\mathcal F},a,s^{\mathcal F})=1\) for all \(a\in\mathcal A^{i}\);
\item \textbf{Restoration:} if \(a^{i}_{k}=m^{i}\) then all probability mass is shifted to states whose \emph{expected time to absorption}---computed under the idle action only and assumed finite---is greater than or equal to that of \(s^{i}_{k}\).
\end{enumerate}

\paragraph{Cost model and shared resources}
Executing \(d^{i}\) is free, \(c_{d^{i}}=0\).  
Restoring consumes one unit of the global \emph{budget} \(B\in\mathbb N_{0}\), i.e., incurs cost \(c_{m^{i}}=1\).  
At any discrete time step no more than \(r\) agents can be restored in parallel (\emph{capacity}).  
Formally, with indicator \(\mathbf 1_{m^{i}}(a^{i})\):
\begin{equation}
\sum_{k=0}^{\infty}\sum_{i=1}^{n}\mathbf 1_{m^{i}}(a^{i}_{k})\le B,\qquad
\sum_{i=1}^{n}\mathbf 1_{m^{i}}(a^{i}_{k})\le r\quad\forall k.\label{eq:constraints}
\end{equation}
\paragraph{Objective.}
Let \(t^{i}_{\mathrm{abs}}\) be the (random) hitting time of \(s^{\mathcal F}\) for agent \(i\).  
The planner seeks a policy \(\pi\) that maximises the worst-case expected lifetime across components:
\begin{equation}
\max_{\pi}\;\min_{i}\;\mathbb E_{\pi}\!\left[t^{i}_{\mathrm{abs}}\right],\label{eq:formulation}
\end{equation}
such that \eqref{eq:constraints} holds.

Because the budget is finite, an infinite planning horizon does not affect optimality; in experiments we truncate at a large \(H\) for practical evaluation.

\begin{figure*}[!htbp]
    \centering
    \includegraphics[width=\linewidth]{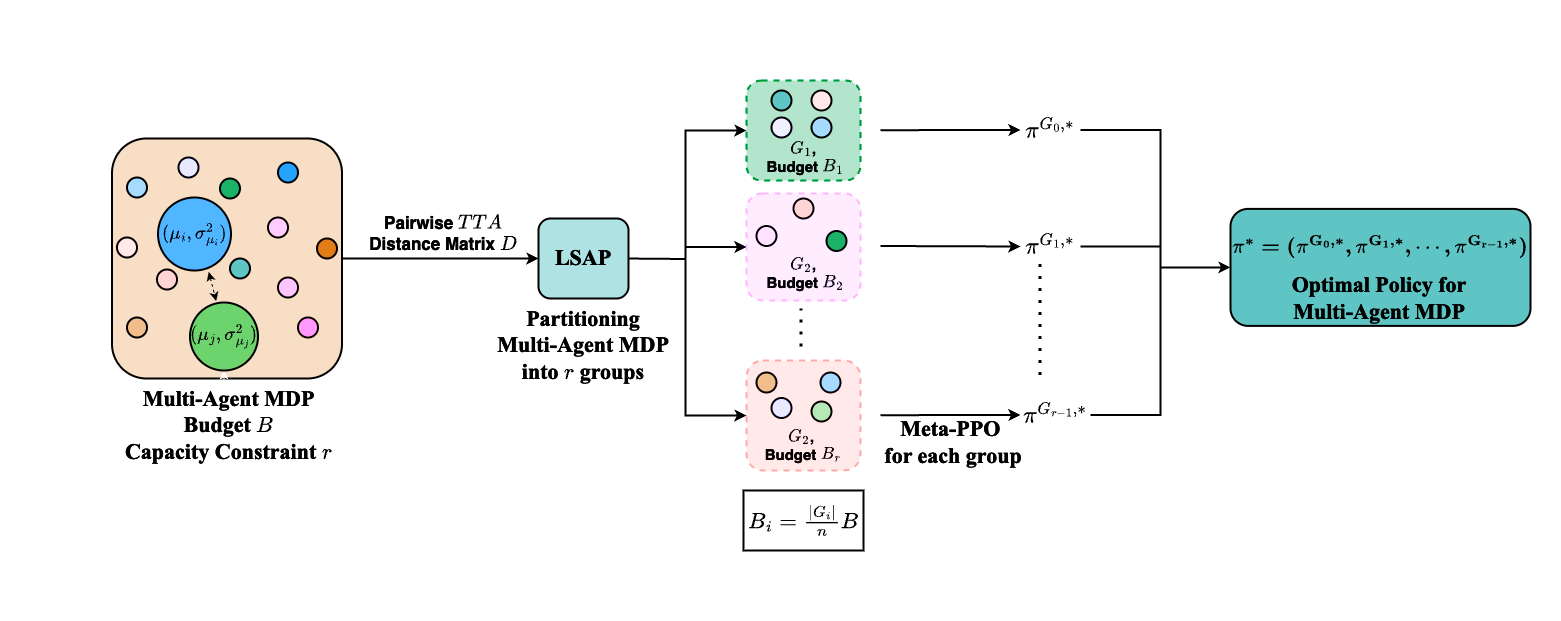}
    \caption{Architectural overview of the proposed approach.}
    \label{fig:arch}
\end{figure*}
\section{Methodology}\label{sec:solution}
Our solution is a two-stage pipeline (Fig.~\ref{fig:arch}).  
\textbf{Stage 1.} We split the \(n\)-agent CB-MA-MDP into exactly \(r\) groups by solving a Linear Sum Assignment Problem (LSAP) followed by pair score sorting. The LSAP cost matrix, built from summary statistics of each local kernel \(T^{i}\), maximises diversity within groups; the resulting partitions have comparable aggregate dynamics, allowing the global budget to be divided proportionally by group size.  
\textbf{Stage 2.} A single meta-reinforcement-learning agent—trained on many group/budget pairs—adapts online to each partition and supplies an approximately optimal policy; composing these policies yields a solution for the original problem.  
The next subsections detail both stages.

\subsection{LSAP Partitioning with Pair-Score Sorting}\label{subsec:lsap}

Solving a CB-MA-MDP exactly via mixed-integer programming \cite{schrijver1998theory,floudas2005mixed} is infeasible for large \(n\) \cite{solozabal2020constrained}.  
We therefore decompose the problem into \(r\) independent \emph{sub-MDPs}.  
Each sub-MDP contains a disjoint subset of agents, inherits a proportional share of the global budget, and has capacity~\(1\).  
Because the global transition kernel factorises, sub-MDPs evolve independently once the partition is chosen, so they can be solved in parallel by the meta-RL stage.

\paragraph{Diversity-maximising assignment}
To construct a balanced partition we first solve a Linear Sum Assignment Problem (LSAP) \cite{burkard1999linear}.  
For each agent \(i\) we pre-compute the mean \(\mu_i\) and variance \(\sigma_{\mu_i}^{2}\) of its \emph{time-to-absorption} (TTA) under the idle policy; these scalar statistics capture how quickly the agent is expected to fail in the absence of restoring actions.  
We define a distance metric
\begin{equation}
D_{ij}\;=\;\sqrt{(\mu_i-\mu_j)^2+(\sigma_{\mu_i}^{2}-\sigma_{\mu_j}^{2})^{2}} \label{eq:cost_matrix}
\end{equation}
and set the LSAP cost matrix to \(C=-D\). The LSAP is formulated as \cite{crouse2016implementing}:
\begin{equation}
\begin{aligned}
    \min_{x} \quad & \sum_{i=1}^{n} \sum_{j=1}^{n} C_{ij} x_{ij} \\
    \text{s.t.} \quad & \sum_{j=1}^n x_{ij} = 1 \quad \forall i,
                      & \sum_{i=1}^n x_{ij} = 1 \quad \forall j, \label{eq:lsap}
\end{aligned}
\end{equation}
where $x_{ij} \in \{0,1\}$ is a binary decision variable indicating whether agent $i$ is assigned to position $j$. Minimising the total cost yields a permutation \(\pi\) that maximises aggregate TTA diversity, resulting in an assignment that separates agents with dissimilar TTA characteristics.

\paragraph{Pair-score sorting and round-robin split}
Each matched pair \((i,\pi(i))\) receives a \emph{pair score} \(s_i=D_{i,\pi(i)}\).  
We sort these scores in descending order (ties broken uniformly at random) and distribute the \emph{second} member of each pair into groups in round-robin fashion:
\[
G_k=\bigl\{\pi(i_{k+\ell r})\bigr\}_{\ell=0}^{\lfloor(n-k-1)/r\rfloor},\quad k=0,\dots,r-1,
\]
where \(i_0,i_1,\dots,i_{n-1}\) is the score-sorted index list.  
This method guarantees that highly dissimilar agents are placed in different groups, giving every sub-MDP a comparable mix of slow- and fast-failing agents. For squared–Euclidean features, the classic variance decomposition implies that increasing within-group scatter reduces between-group dispersion \cite{Jain1999DataClustering}, making groups’ aggregate failure behavior more alike. This uniform diversity delays simultaneous failures and justifies allocating the global budget to groups in proportion to their sizes.  
The result is \(r\) smaller CB-MA-MDPs, each with capacity \(1\). Because these sub-problems now have nearly identical failure profiles,
a single meta-RL policy transfers across groups with minimal
adaptation. For completeness, we also evaluated a simple variant that assigns whole LSAP pairs to groups in round-robin order (“pair–round–robin”); this method admits a per-group diversity lower bound, but empirically underperforms the above proposed scheme across all scenarios.

\subsection{Meta-RL for Each Sub-MDP}\label{subsec:metarl}

The LSAP stage (Section~\ref{subsec:lsap}) yields \(r\) independent sub-MDPs.  
We now learn an approximately optimal policy for each sub-MDP via meta-reinforcement learning (RL).

\paragraph{State}
Following the budgeted-MDP formalism \cite{vora2023welfare}, the residual budget is appended to the physical state.  
For group \(G_q\) with \(m_q\) agents and budget \(B_q\),
\[
s^{G_q}_k=\begin{bmatrix}
            s^{1}_k & \dots & s^{m_q}_k\\[2pt]
            B_k     & \dots & B_k
          \end{bmatrix}^{\!\top}\!,
\]
where \(s^{i}_k\) is the physical state of agent \(i \in G_q\) and \(B_k\le B_q\).  
The \emph{column-stacked} layout separates the budget row from the physical state rows. Empirically this improves credit assignment compared with a flat \([s^{1}_k,\dots,s^{m_q}_k,B_k]\) vector \cite{lesort2018state}.

\paragraph{Action space}
Because each sub-MDP has capacity~1, the RL policy need only choose  
\(\texttt{NOOP}\) (idle) \emph{or} \(\texttt{REPAIR}(i)\) for exactly one \(i\in\{1,\dots,m_q\}\), i.e., $m_q + 1$ discrete actions.
This replaces the length-\(m_q\) Boolean action vector and enforces the capacity constraint by construction.

\paragraph{Reward}
\[
R(s_k,a_k)=
\begin{cases}
r_1 < 0 & \text{if } B_k<0,\\
-(H-k) & \text{if } \exists\,i:\,s^{i}_k=s^{\mathcal F},\\
k-\alpha\,s^{i}_k & \text{if } a_k=\texttt{REPAIR}(i),\\
k & \text{if } a_k=\texttt{NOOP},
\end{cases}
\]
where $0<\alpha<1$. The large penalty \(r_1\) discourages budget violations; the \(-(H-k)\) term penalises early absorption; the linear term \(k\) rewards survival; and the \(-\alpha s^{i}_k\) offset deters unnecessary restorations. A few principled alternatives to the above proposed heuristic reward function include (i) Lagrangian constrained-RL reward with a dual on budget/capacity costs \cite{achiam2017cpo} and (ii) an interior-point (log-barrier) pacing term with potential-based shaping reward \cite{boyd2004convex}; in ablations both performed comparably or slightly worse than our heuristic reward.

\paragraph{Meta-training and deployment.}
A single Proximal Policy optimization (PPO) policy \cite{schulman2017proximal} is meta-trained across a distribution of tuples \((G_q,B_q)\). During meta-training the PPO policy samples random \((G_q,B_q)\) tuples, rolls out \(H\) steps, and updates shared parameters similar to the meta-RL loop in $\text{RL}^2$ \cite{duan2016rl}.  
At deployment the learned network is cloned for each group and fine-tuned for a few gradient steps, yielding policies \(\pi^{G_q}_{\!*}\).  
The global controller applies
\[
\pi^\ast(s_k)=\bigl(\pi^{G_1}_{\!*}(s^{G_1}_k),\dots,\pi^{G_r}_{\!*}(s^{G_r}_k)\bigr),
\]
mapping local actions back to the original agent indices recorded during LSAP partitioning.  
Although \(\pi^\ast\) is not provably optimal, experiments confirm that it satisfies both budget and capacity limits and outperforms heuristic baselines on scenarios with up to 1000 agents.

\section{Implementation and Evaluation}\label{sec:experiments}

We evaluate our two-stage pipeline on a synthetic yet realistic maintenance scenario involving a large team of industrial robots.  The task is to synthesise an approximately optimal policy for a CB-MA-MDP with up to one thousand agents.  Performance is benchmarked against a suite of learning- and optimization-based baselines, and a complexity study quantifies scalability.

\paragraph{Scenario description}
The swarm contains \(n\) heterogeneous robots—assembly arms, pick-and-pack units, welding cells—serviced by \(r\) repair technicians.  Each robot’s health is represented by a \emph{Condition Index} (CI) that ranges from 100 (perfect) to 0 (failed) \cite{grussing2006condition}.  Motivated by infrastructure-deterioration studies \cite{grussing2006condition},
stochastic wear for each robot is modeled with a
\emph{discrete Weibull deterioration kernel}:
starting from CI \(h\!\in\!\{100,\dots,0\}\) the one-step
probability of landing in CI \(h'\le h\) is
\begin{equation*}
  P\bigl(\mathrm{CI}_{k+1}=h'\mid\mathrm{CI}_{k}=h\bigr)=\frac{f_{k_i,\lambda_i}(h-h'+1)}
       {\displaystyle\sum_{d=0}^{h} f_{k_i,\lambda_i}(d+1)}
\end{equation*}
where \(f_{k,\lambda}\) is the continuous Weibull probability density function \cite{kizilersu2018weibull}. Shape and scale parameters, $k_i, \lambda_i$, are drawn once per
agent from the practical ranges
\(k_i\!\sim\!\mathcal U[1,7]\) and \(\lambda_i\!\sim\!\mathcal U[25,70]\)
(CI units); these intervals are typical for electromechanical components
\cite{Abernethy2006Reliability, Nelson2004LifeData}. At any decision step each technician may restore at most one robot to full CI at a cost of one budget unit.  The swarm is considered non-operational once \emph{any} robot reaches CI \(=0\).

\paragraph{Global constraints}
A mission budget of \(B\) repair units is available over a planning horizon of \(H=100\) steps, and at most \(r\) robots can be repaired simultaneously.  All robots start at CI \(=100\).  The planner must choose, at every step, which subset of robots to repair so as to maximise the expected time until the first failure while never exceeding the budget and capacity limits.

\paragraph{MDP formulation}
The problem instantiates the CB-MA-MDP defined in Section~\ref{sec:prob}.  Each agent’s state is its current CI.  For \(n=1000\) and \(r=300\) the exact mixed-integer formulation contains roughly \(10^{5}\) binary decision variables over \(H=100\) steps, illustrating the need for the scalable partition-and-meta-RL strategy evaluated below.

We report results across multiple \((n,r)\) pairs to assess solution quality, robustness, and runtime. For the policy-level performance analysis (Section~\ref{subsec:policyperf}), the budget is $B=10n$ units for each $(n,r)$ scenario.

\subsection{Partition-Quality Analysis}\label{subsec:swarm_partitioning}
Partitioning is \emph{not} part of the CB-MA-MDP definition; we introduce it purely as a scalability device.  
Any algorithm that optimises the full joint action space (e.g.\ the ILP baseline) becomes infeasible once \(n\gtrsim 10\) because the number of binary variables explodes.  
Our strategy is to solve many small, independent sub-MDPs instead.  
For the meta-PPO stage to generalise across groups and perform well, the partitions should (i) expose sufficient diversity \emph{within} each group and (ii) make the groups \emph{similar} to one another so that a single policy can transfer.  
The average in-group distance \(\bar d\) therefore serves as a proxy for downstream policy quality. We compare three partioning strategies:

\begin{itemize}[leftmargin=*]
\item \textbf{Random grouping} – each agent is assigned uniformly at random to one of the \(r\) groups.
\item \textbf{MILP diversity maximisation} – a mixed-integer linear program that assigns agents to groups by maximising the total in-group TTA distance \(\sum_{a<b}D_{ab}y_{ab}\).  
      Binary \(x_{ij}\) variables select the group for agent \(i\) (\(\sum_j x_{ij}=1\)); auxiliary \(y_{ab}\) indicate if agents \(a,b\) share a group.  
      The constraint matrix is totally unimodular, so we first solve the LP relaxation; if an integer gap remains after twice the LSAP runtime, the relaxed solution is greedily rounded, otherwise the exact optimum is returned.
\item \textbf{LSAP + pair-score sorting} (ours) – the algorithm of Section~\ref{subsec:lsap}, solved with the Jonker–Volgenant assignment Python routine \texttt{linear\_sum\_assignment} \cite{crouse2016implementing}.
\end{itemize}
\definecolor{bestcell}{RGB}{224,247,255}
\newcommand{\best}[1]{\cellcolor{bestcell}\textbf{#1}}
\begin{table}[!htbp]
\centering
\small
\setlength{\tabcolsep}{4pt}
\renewcommand{\arraystretch}{1.12}
\begin{tabular}{@{}ccccc@{}}
\toprule
\(n\) & \(r\) &
\textbf{Random} & \textbf{MILP} & \textbf{LSAP}\\
\midrule
10   & 3   & 31.15 & \best{39.33} & 38.40  \\
20   & 5   & 54.31 & \best{67.41} & 64.47 \\
50   & 15  & 59.50 & 31.15 & \best{66.00} \\
100  & 25  & 67.15 & -- & \best{74.35}\\
500  & 150 & 75.77 & -- & \best{77.50}\\
1000 & 300 & 76.46 & -- & \best{76.86}\\
\bottomrule
\end{tabular}
\caption{Average in-group TTA diversity \(\bar d\) (higher $\uparrow$ is better).  
“--’’ denotes that the MILP failed to return a feasible solution.  
The highest diversity per row is shaded.}
\label{tab:partition}
\end{table}
\paragraph{Evaluation metric}
For each partition we compute the average in-group distance  
\[
\bar d=\frac1r\sum_{q=1}^{r}\;\;
        \frac{1}{|G_q|(|G_q|\!-\!1)}
        \sum_{\substack{i,j\in G_q\\i<j}}\!D_{ij}.
\]
 Larger \(\bar d\) means each group contains a more heterogeneous set of agents, which in turn yields groups with similar aggregate failure profiles.
\paragraph{Experimental protocol}
Means \(\mu_i\) and variances \(\sigma_{\mu_i}^2\) are estimated from \(10^{3}\) Monte-Carlo runs per agent.  
The MILP is solved with GUROBI 10.0; if no integer solution is obtained within \textbf{twice the LSAP + meta-PPO wall-clock time}, we report the LP-rounded solution.  
A “--’’ entry indicates that even the relaxation exceeded the time limit.
\definecolor{bestrow}{RGB}{224,247,255}
\begin{table*}[!htbp]
\centering
\small
\setlength{\tabcolsep}{4pt}
\renewcommand{\arraystretch}{1.13}
\begin{tabular}{@{}lcccccccc@{}}
\toprule
\multirow{2}{*}{\textbf{Method}} &
\multicolumn{2}{c}{$(n,r)$ = \((2,1)\)} &
\multicolumn{2}{c}{$(n,r)$ = \((10,3)\)} &
\multicolumn{2}{c}{$(n,r)$ = \((100,30)\)} &
\multicolumn{2}{c}{$(n,r)$ = \((1000,300)\)}\\
\cmidrule(lr){2-3}\cmidrule(lr){4-5}\cmidrule(lr){6-7}\cmidrule(l){8-9}
& $\bar{t}_{\mathrm{abs}}\!\pm\!\sigma$ & $\bar{U}\!\pm\!\sigma$
& $\bar{t}_{\mathrm{abs}}\!\pm\!\sigma$ & $\bar{U}\!\pm\!\sigma$
& $\bar{t}_{\mathrm{abs}}\!\pm\!\sigma$ & $\bar{U}\!\pm\!\sigma$
& $\bar{t}_{\mathrm{abs}}\!\pm\!\sigma$ & $\bar{U}\!\pm\!\sigma$\\
\midrule
\textbf{ILP}                 & \best{46.1$\pm$3.4} & 19.5$\pm$2.6 &
33.4$\pm$8.4 & 92.2$\pm$23.3 & -- & -- & -- & --\\
\textbf{Vanilla PPO}         & 44.4$\pm$9.9 & 18.3$\pm$2.8 &
33.8$\pm$18.2 & 63.1$\pm$35.0 &
3.2$\pm$0.7 & 0.0$\pm$0.0 &
2.0$\pm$0.2 & 361.5$\pm$59.9\\
\textbf{GA}        & 41.1$\pm$1.5 & 20.0$\pm$0.0 &
31.5$\pm$10.5 & 86.1$\pm$26.9 &
3.1$\pm$1.1 & 93.0$\pm$34.1 &
1.0$\pm$0.0 & 300.0$\pm$0.0\\
\textbf{Auction Heuristic}             & 42.4$\pm$4.5 & 19.5$\pm$2.4 &
33.5$\pm$9.6 & 90.6$\pm$24.8 &
8.3$\pm$7.6 & 249.1$\pm$227.6 &
1.1$\pm$0.3 & 324.0$\pm$81.4\\
\textbf{RP-PPO}    & 44.4$\pm$9.9 & 18.3$\pm$2.8 &
72.3$\pm$23.9 & 27.1$\pm$5.5 &
37.7$\pm$16.0 & 25.8$\pm$9.4 &
32.6$\pm$15.5 & 23.5$\pm$10.3\\
\textbf{MP-PPO}      & 44.4$\pm$9.9 & 18.3$\pm$2.8 &
\best{80.6$\pm$24.0} & 28.0$\pm$7.3 & -- & -- & -- & --\\
\textbf{LSAP + Meta-PPO (Ours)}     & 44.4$\pm$9.9 & 18.3$\pm$2.8 &
\best{80.4$\pm$24.0} & 28.8$\pm$8.4 &
\best{80.8$\pm$21.0} & 28.8$\pm$7.5 &
\best{35.0$\pm$20.6} & 23.1$\pm$10.6\\
\bottomrule
\end{tabular}
\caption{Mean survival time \(\bar{t}_{abs}\) (higher $\uparrow$) and mean cumulative repairs \(\bar{U}\) under a fixed budget \(B=10n\).  Best \(\bar{t_{abs}}\) in each scenario is highlighted; “--’’ denotes that ILP / MILP failed to return a feasible solution within twice the runtime of LSAP + meta-PPO.}
\label{tab:policy_perf}
\end{table*}
\begin{figure}[!htbp]
  \centering
  \includegraphics[width=\linewidth]{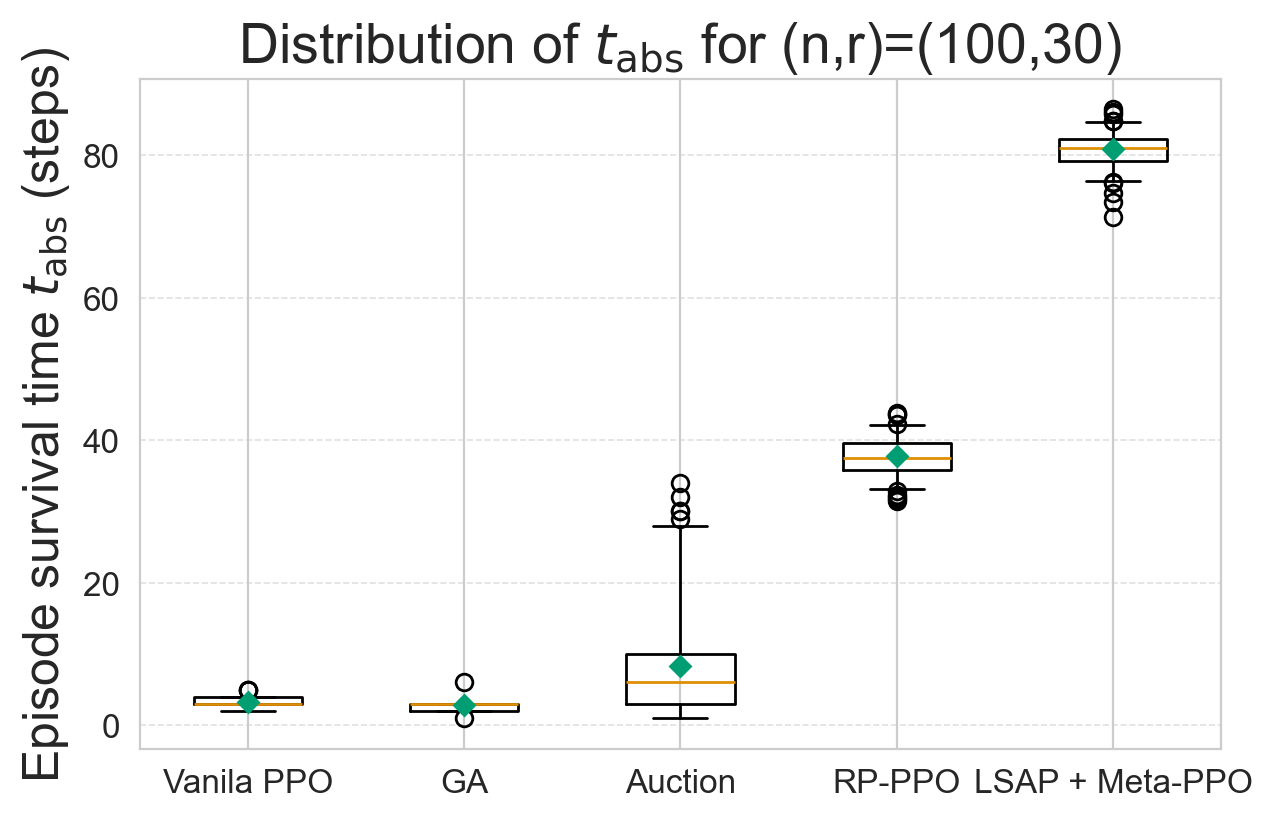}
  \caption{Distribution of episode survival time \(T_{\mathrm{abs}}\) for \((n,r)=(100,30)\) across methods (100 runs each). Boxes show the inter-quartile range, whiskers represent the 5--95th percentiles; markers denote means and circles denote outliers.}
  \label{fig:tabs_boxplot}
\end{figure}
\paragraph{Discussion}
LSAP outperforms the random grouping baseline in all scenarios and approaches the MILP optimum when the latter converges.  
For $(10,3)$ and \((20,5)\) the MILP attains the highest \(\bar d\) but requires twice the LSAP + meta-PPO runtime; for $(50,15)$ the LP–relaxed MILP does not reach integrality within the time limit. 
Our greedy rounding step therefore returns many \emph{singleton} groups (size 1), each contributing zero in-group distance, 
which explains the unexpectedly low diversity score.
Beyond \(n=100\) the MILP fails to find a feasible solution within the time limit \cite{schrijver1998theory}, while LSAP completes in milliseconds (see Table~\ref{tab:stepwise_complexity}) and achieves the highest \(\bar d\).


\textbf{Asymptotic intuition.}
Theorem~\ref{lem:partition_quality} in Appendix~\ref{subsec:proof} shows that \emph{if} distances were
i.i.d.\ Gaussian, any balanced partition becomes nearly as diverse as
any other once the group size $k$ grows faster than $\ln n$ for large $n$.
Although our actual TTA distances are computed from  Weibull parameters and therefore do not satisfy the assumptions of the theorem, the empirical results in
Table~\ref{tab:partition} display the same behavior: for $n{=}1000$ case, the diversity gap between
LSAP and random grouping shrinks to below 1 \%.  Thus the
concentration intuition, while not rigorously proved for our experimental setting, is observed empirically.

\subsection{Policy-Level Performance}\label{subsec:policyperf}

We compare the maintenance policies returned by our proposed approach with six baselines that span exact optimization, meta-heuristics, and learning without partitioning:
\begin{enumerate}[leftmargin=*]
\item \textbf{Integer Linear Programming (ILP)} – exact integer linear program for \eqref{eq:formulation} solved with GUROBI, with the same time cap as MILP (Section~\ref{subsec:swarm_partitioning}).  
\item \textbf{Vanilla PPO} – one PPO network trained on the full CB-MA-MDP.  
\item \textbf{Genetic Algorithm (GA)} \cite{Goldberg1989GA} – rank selection, two-point crossover, bit-flip mutation; fitness equals next-step total health.  
\item \textbf{Auction Heuristic} \cite{Dias2006TraderBots} – each agent bids a linear function of failure risk; the \(r\) highest bids receive repairs each step.  
\item \textbf{Random Partition + Meta-PPO (RP-PPO)} - robots are randomly assigned to groups, and a meta-PPO agent is then used to determine the repair policy, as described in Section~\ref{sec:rl}.  
\item \textbf{MILP Partition + Meta-PPO (MP-PPO)} – diversity-maximizing MILP (LP relaxation, rounding) followed by meta-PPO.  
\end{enumerate}

\emph{Evaluation metrics:}
Our primary metric is the \emph{survival time}
\(t_{\mathrm{abs}}\), the number of decision steps until the first
absorption (i.e., until any robot’s CI reaches~0); if no failure occurs
within the horizon \(H=100\), we set \(t_{\mathrm{abs}}=H\).
As a secondary, efficiency-oriented metric we report the
\emph{cumulative repairs}
\(U=\sum_{t=1}^{t_{\mathrm{abs}}}\sum_{i=1}^{n}\mathbf{1}\{a_i(t)=\text{repair}\}\).
Table~\ref{tab:policy_perf} presents the performance of the proposed approach and the six baselines. Every algorithm runs for 100 episodes on four scenarios: \((n,r)\in\{(2,1),(10,3),(100,30),(1000,300)\}\).  
ILP and MILP receive \emph{double} the wall-clock time consumed by our approach; if they fail to return a feasible policy the entry is “--”.  
We report the mean survival time \(\bar{t}_{abs}\) (higher is better) and the mean cumulative number of repairs \(\bar{U}\) over the 100-step horizon, averaged over 100 runs; \(\sigma(\cdot)\) denotes the unbiased sample standard deviation. A larger \(\bar{t}_{\mathrm{abs}}\) is always desirable;  
\(\bar{U}\) is interpreted \emph{relative} to \(\bar{t}_{\mathrm{abs}}\):  
a smaller $\bar{U}$ for a high $\bar{t}_{abs}$ indicates more efficient budget pacing,  
whereas a higher $\bar{U}$ for a small $\bar{t}_{abs}$ signals poor allocation. From a managerial perspective, for comparable $\bar t_{\mathrm{abs}}$, a smaller $\bar U$ indicates more even technician utilization and lower average overtime risk, whereas a large $\bar U$ accompanied by a small $\bar t_{\mathrm{abs}}$ suggests front-loaded or mistimed repairs that create budget spikes and staffing stress. Figure~\ref{fig:tabs_boxplot} visualizes the full distribution of
\(t_{\mathrm{abs}}\) for \((n,r)=(100,30)\), complementing
Table~\ref{tab:policy_perf}.

\noindent\emph{\textbf{Discussion:}} For the smallest scenario $(n,r) = (2,1)$ ILP attains the global optimum, as expected for such a small search space; baselines (b), (e) and (f) perform the same as our approach and none of them match the optimum. This performance similarity is because the (2,1) scenario doesn't require any partitioning.
Once the swarm grows beyond ten agents, ILP and MILP exceed the time budget, and vanilla PPO, GA, and Auction heuristics degrade sharply as the joint action space explodes. At $(10,3)$, the ILP solver timed-out, yielding a feasible but sub-optimal solution; with unlimited time we expect it to reach the optimum.
In contrast, both partition-then-PPO pipelines preserve high swarm operational time: MILP + PPO and our LSAP + PPO achieve comparable survival at \((10,3)\), but LSAP scales to \((100,30)\) and \((1000,300)\) whereas MILP times out.  At \(n=100\) LSAP prolongs operation by a factor of ten relative to the strongest non-partitioned heuristic (Auction), while requiring an order-of-magnitude fewer repairs. For the largest case \((1000,300)\) LSAP and Random partitioning yield similar performance.  Such a feature is indeed expected for large $n$, as previously discussed in Section~\ref{subsec:swarm_partitioning}.
\begin{figure}[!htbp]
  \centering
  \includegraphics[width=\linewidth]{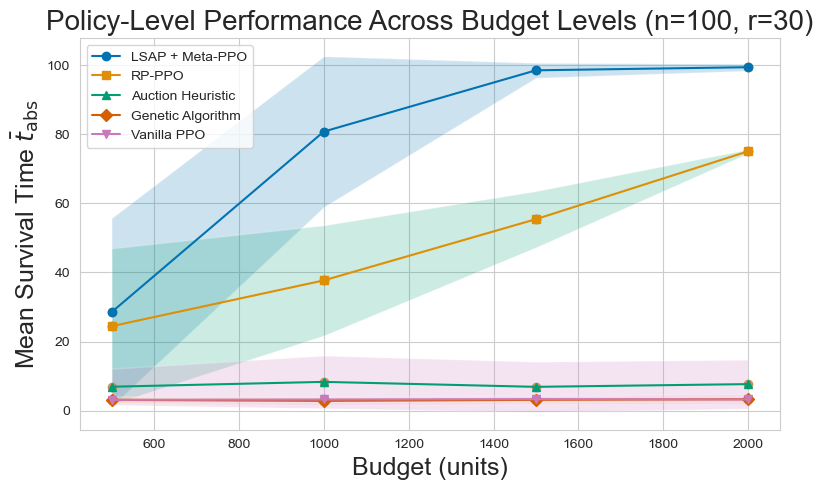}
  \caption{Budget sensitivity for \((n,r)=(100,30)\): mean survival time vs.\ total budget \(B\).  Shaded bands are $\pm \sigma$ over 100 runs; ILP/MILP not shown (timeout).}
  \label{fig:budget_sensitivity}
\end{figure}

\emph{\textbf{Budget sensitivity.}}
To verify that our choice \(B=10n\) does not unduly favor the proposed approach, we re-evaluate the methods on \((n,r)=(100,30)\) for budgets \(B\in\{5n,10n,15n,20n\}\).
Figure~\ref{fig:budget_sensitivity} reports mean survival (\(\pm\)1 standard deviation) over 100 runs. 
LSAP\,+\,meta-PPO increases $\bar{t}_{abs}$ monotonically with \(B\) and saturates near the 100-step horizon, whereas non-partitioned baselines (Vanilla PPO, GA, Auction) plateau at substantially lower values; even RP-PPO improves with budget but remains well below our method.
ILP and MILP are omitted in this plot because neither returned a feasible policy within the $2\times$(LSAP+meta-PPO) runtime cap. 
These trends demonstrate that our approach is robust across a wide range of budgets.
\begin{table}[!htbp]
    \centering
    \begin{tabular}{c c c c}
        \toprule
        \textbf{$n$} & \textbf{$r$} & \textbf{LSAP Partitioning} & \textbf{Meta-PPO} \\
        \midrule
        2 & 1 & 0.5001 & 4.4594\\
        5 & 2 & 0.3905 & 3.4927\\
        10 & 3 & 0.4464 & 4.9642\\
        20 & 6 & 0.3992 & 8.1406\\
        50 & 15 & 0.3956 & 17.4025\\
        100 & 30 & 0.8295 & 32.1637\\
        500 & 150 & 0.6516 & 182.0258\\
        1000 & 300 & 1.2324 & 290.1350\\
        \bottomrule
    \end{tabular}
    \caption{Time taken (in seconds), averaged over 10 runs, for running each step of the proposed approach for varying number of robots and repair technicians.}
    \label{tab:stepwise_complexity}
\end{table}
\begin{figure}[!htbp]
    \centering
    \includegraphics[width=\linewidth]{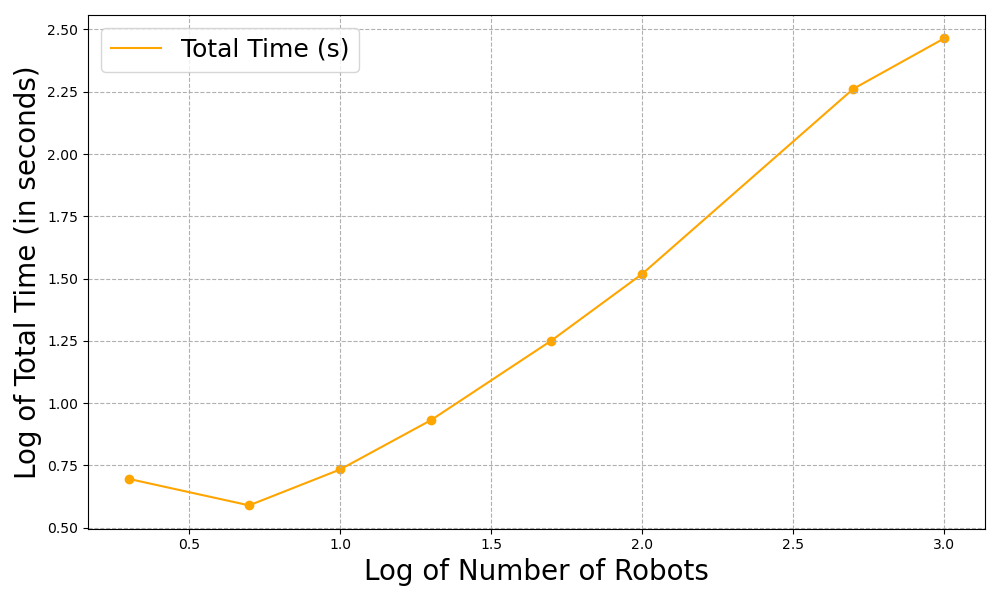}
    \caption{Log-log plot showing the computational complexity of the proposed approach for varying number of robots. $(n,r)$ values are the same as Table~\ref{tab:stepwise_complexity}.}
    \label{fig:complexity}
\end{figure}

\subsection{Computational Complexity Analysis}
Finally, we perform a computational complexity analysis to demonstrate the scalability of the proposed approach across different $(n,r)$ pairs.
The experiments were conducted in Python on a laptop running MacOS with an M2 chip @3.49GHz CPU and 8GB RAM. Table~\ref{tab:stepwise_complexity} summarizes the time taken (in seconds) for each step of the proposed approach for varying number of robots and repair technicians. We observe that the time taken for the LSAP-based partitioning step is negiligible in comparison to the time required for generating policies using the meta-PPO agent, especially as the number of robots increases. Since the second step involves applying a pre-trained meta-PPO model to each group, the time complexity for this step scales linearly with the number of robots in each group. Therefore, the overall complexity of our algorithm to expected to be linear in the number of robots, i.e., $O(n)$. This hypothesis is confirmed by the log-log plot shown in Figure~\ref{fig:complexity}, which depicts the computational complexity of the proposed approach as the number of robots increases. The trend in the log-log plot demonstrates the linear scalability of our algorithm with respect to the number of robots. Figure~\ref{fig:heatmap} presents a heatmap that captures the variation in computational time for different combinations of $(n,r)$ pairs. We observe that the time taken to run the proposed approach is more sensitive to changes in the number of robots than to changes in the number of repair technicians. This finding indicates that the overall computational complexity is dominated by the number of robots, with the impact of the number of repair technicians being relatively small.
\begin{figure}[h]
    \centering
    \includegraphics[width=\linewidth]{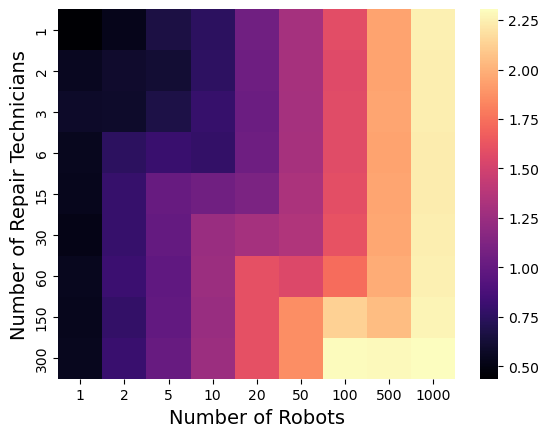}
    \caption{Heatmap showing the time taken (in seconds) for the proposed approach, across different numbers of robots and repair technicians. Values are plotted on a logarithmic scale to better capture the variation.}
    \label{fig:heatmap}
\end{figure}

\section{Conclusions}

We introduced a two-stage, partition-and-meta-RL algorithm for \emph{capacity- and budget-constrained multi-agent MDPs}.  
The key difficulty—an exponentially large joint action set induced by global budget and simultaneous-action limits—is overcome by (i) partitioning the \(n\)-agent system into \(r\) groups by solving a Linear Sum Asignment Problem (LSAP) that maximises transition-diversity, then (ii) training a single meta-PPO policy that adapts rapidly to each sub-MDP.  
A proportional budget split renders the sub-problems independent and solvable in parallel. Empirical results on a large-scale robot-maintenance benchmark show that the LSAP + meta-PPO pipeline consistently outperforms integer-programming, vanilla PPO, and several meta-heuristic baselines; the advantage widens with swarm size.  
Runtime grows near-linearly with \(n\), confirming that the method is suitable for large systems. From an operational standpoint, the method’s \emph{linear scalability} allows maintenance planners to scale from tens to thousands of robots without re-tuning; its near-uniform \emph{budget pacing} and consistently low technician usage translate into fewer overtime shifts and, consequently, higher line uptime and lower labour costs in real‐world factories. Future work will explore hierarchical budget structures and richer capacity models.

\bibliographystyle{IEEEtran}
\bibliography{references}

\appendix
\section{Appendix}

\subsection{Almost Sure Asymptotic Contraction of Partition Averages}\label{subsec:proof}


\begin{theorem}
\label{lem:partition_quality}
Let $X \in \mathbb{R}^{nk \times nk}$ be a symmetric random matrix where the entries $\{ X_{ij} \mid i \leq j \}$ are independently and identically distributed (i.i.d.) standard normal random variables, i.e., $X_{ij} \sim \mathcal{N}(0,1)$ for $i \leq j$, and $X_{ji} = X_{ij}$ for $i > j$. Let $\mathcal{G} = (V, E)$ be a complete graph with $nk$ nodes, where $V$ denotes the set of vertices and $E$ denotes the set of edges, and the adjacency matrix is given by $X$. 

Partition the node set $V$ into $n$ disjoint subsets $V_1, V_2, \dots, V_n$, each of size $k$. For each $t \in \{1,2,\dots,n\}$, let $E_t$ be the edges in the subgraph induced by $V_t$. Let $S_t$ denote the average of the edge weights within subgraph $V_t$. 

Then, for any fixed $\epsilon > 0$, if $k = \omega(\ln n)$ (i.e., $k$ grows faster than $\ln n$), the probability that there exist $t \neq l$ such that $|S_t - S_l| > \epsilon$ tends to zero as $n \to \infty$. That is,
\begin{equation}
    \lim_{n \to \infty} \max_{V_1, V_2, \cdots, V_n} P\left( \exists\, t \neq l \text{ such that } |S_t - S_l| > \epsilon \right) = 0.
\end{equation}
\end{theorem}

\begin{proof}
The number of edges in $E_t$ is
\begin{equation}
    |E_t| = \binom{k}{2} = \frac{k(k - 1)}{2}.
\end{equation}
Also, since $S_t$ is the average of the edge weights within the subgraph induced by $V_t$, we have:
\begin{equation}
    S_t = \frac{1}{|E_t|} \sum_{(i,j) \in E_t} X_{ij} = \frac{1}{\binom{k}{2}}\sum_{(i,j)\in E_t}X_{ij}. \label{eq:st}
\end{equation}
For each $t$, $S_t$ is the average of $|E_t| = \binom{k}{2}$ i.i.d.\ standard normal random variables $X_{ij} \sim \mathcal{N}(0,1)$. Therefore, $S_t$ is normally distributed with mean zero and variance
\begin{equation}
    \sigma^2 = \frac{1}{|E_t|} = \frac{2}{k(k - 1)}.
\end{equation}
Similarly, the difference between any two such averages $S_t$ and $S_l$ (for $t \neq l$) is also normally distributed with mean zero and variance $2\sigma^2$:
\begin{equation}
    S_t - S_l \sim \mathcal{N}\left( 0, 2\sigma^2 \right).
\end{equation}
Using the Chernoff bound we have:
\begin{equation}
    \begin{split}
        P\left( |S_t - S_l| > \epsilon \right) &\leq P\left(|S_t| > \frac{\epsilon}{2}\right) + P\left(|S_l| > \frac{\epsilon}{2}\right) \\
        &\leq 4\exp\left(-\frac{\epsilon^2k(k-1)}{32}\right).
    \end{split}
\end{equation}
There are at most $\binom{nk}{k}$ possible choices for $S_t$. Using the upper bound on the binomial coefficient, we have:
\begin{equation*}
    \binom{nk}{k} \leq \left(\frac{enk}{k}\right)^k = \exp(k \ln (en)).
\end{equation*}
Therefore, the total number of $(S_t, S_l)$ pairs is at most
\begin{equation}
     \binom{nk}{k}^2 \leq \exp\left(2 k \ln (e n)\right).
\end{equation}
Applying the union bound over all partitions and all pairs, we obtain
\begin{equation}\label{eq:simple}
\begin{split}
\max_{V_1,\dots,V_n} P\!\bigl(\exists\,t\neq l:\,|S_t-S_l|>\epsilon\bigr)
&\le 4\,\binom{nk}{k}^{\!2} e^{-\frac{\epsilon^{2}k(k-1)}{32}}\\
&\le 4\,e^{\,2k(1+\ln n)}\,e^{-\frac{\epsilon^{2}k(k-1)}{32}}.
\end{split}
\end{equation}

The exponent in \eqref{eq:simple} tends to $-\infty$ if $k$ grows faster than $\ln n$, i.e., $k = \omega(\ln n)$. This means that the right-hand side (RHS) of \eqref{eq:simple} goes to 0.

Thus, if $k$ grows faster than $\ln n$, the probability that any two group averages $S_t$ and $S_l$ differ by more than $\epsilon$ tends to zero. This implies that, asymptotically, all $S_t$ are approximately equal with high probability regardless of how the partitions $V_t$ are formed. 
\end{proof}

\end{document}